%% file: lower_bound_decDNNF.tex
\documentclass[a4paper,11pt]{article}
\usepackage[utf8]{inputenc}

\usepackage{microtype}
\usepackage{amsthm}
\usepackage{amssymb}
\usepackage{amsmath}
\usepackage{ stmaryrd }
\usepackage{ucs,algorithmic,algorithm}
\usepackage{complexity}
\usepackage{graphicx}
\usepackage{tikz}
\usepackage{tkz-graph}
\usepackage{xifthen}

\newcommand{\var}{\mathsf{var}}

\newcommand{\ptw}{\mathsf{ptw}}
\newcommand{\tw}{\mathsf{tw}}
\newcommand{\itw}{\mathsf{itw}}

\newcommand{\aobdd}{\wedge\text{-OBDD}}
\newcommand{\adobdd}{\wedge_d\text{-OBDD}}
\newcommand{\afbdd}{\wedge\text{-FBDD}}

\newcommand{\calF}{\mathcal{F}}
\newcommand{\sSAT}{\#\SAT}
\newcommand{\calG}{\mathcal G}

\title{Non-FPT lower bounds for structural restrictions of decision DNNF}

\newtheorem{theorem}{Theorem}

\newtheorem{corollary}[theorem]{Corollary}
\newtheorem{lemma}[theorem]{Lemma}

\begin{document}

\author{
Andrea Calì \and Florent Capelli \and Igor Razgon
}

\maketitle

\begin{abstract}
We give a non-FPT lower bound on the size of structured decision DNNF and OBDD with decomposable AND-nodes representing CNF-formulas of bounded incidence treewidth. Both models are known to be of FPT size for CNFs of bounded primal treewidth. To the best of our knowledge this is the first parameterized separation of primal treewidth and incidence treewidth for knowledge compilation models.
\end{abstract}

\section{Introduction}
\label{sec:intro}

Counting the satisfying assignments of a given CNF-formula, a problem known as $\sSAT$, is a central task to several areas of computer science such as probabilistic reasoning~\cite{Roth96,BacchusDP03} or probabilistic databases~\cite{BeameLRS14,BeameLRS13,JhaS13}. Many tools such as {\tt c2d}~\cite{OztokD14}, {\tt Cachet}~\cite{sang04}, {\tt sharpSAT}~\cite{thurley06} or more recently {\tt D4}~\cite{LagniezM17} have been developed to solve this problem in practice. They are all based on the same algorithm, called {\em exhaustive DPLL}~\cite{BacchusDP03}, which relies on the fact that for a given CNF $F$ and a variable $x$, $\#F = \#F[x \mapsto 0]+\#F[x \mapsto 1]$ together with the fact that if  $F = G \land H$ with $G,H$ having disjoint variables, then $\#F = \#G \cdot \#H$. These algorithms also use a cache, that is, they sometimes store the number of satisfying assignments of a subformula to reuse it later in the computation. The performance of these implementations mainly depends on the heuristics that are chosen concerning the order variables are eliminated and the caching policy. It has been observed by Darwiche and Huang~\cite{HuangD05} that the trace of these algorithms can be seen as a very special kind of Boolean circuits -- known as decision Decomposable Normal Form (decDNNF) in the field of Knowledge Compilation -- computing the input CNF. The size of this circuit is proportional to the runtime of the algorithm. This observation leads to the interesting fact that proving lower bounds on the size of decDNNF computing a given CNF-formula gives a lower bound on the runtime of these algorithms, independently from the implementation.

Another successful line of research concerning the complexity of $\sSAT$ is to study the so-called {\em structural restrictions} of CNF-formulas in order to find classes of formulas for which the number of satisfying assignments can be computed in polynomial time. Such classes, usually referred as {\em tractable classes}, are often defined by specifying a structure describing the way the clauses and the variables interact in the formula. This interaction is usually represented with a graph, the {\em incidence graph}, whose vertices are the variables and the clauses and where a variable $x$ and a clause $C$ are connected by an edge if $x$ is a variable of $C$. Many polynomial time algorithms have been designed to solve $\sSAT$ when the incidence graph belongs to some interesting class~\cite{SlivovskyS13,SaetherTV14,CapelliDM14,BraultCM15,PaulusmaSlivovskySzeider16,Capelli17}. Most of the time however, these algorithms work in a very different way than exhaustive DPLL and are relevant only when the input belongs to a particular class of formulas. Moreover, these algorithms usually start with a costly step of decomposing the incidence graph, making them currently of little interest in practice. An interesting question is thus to understand how the existing tools behave on these instances that are easy to solve in theory.

Unfortunately, the performance of exhaustive DPLL -- and thus, of practical tools for $\sSAT$ -- is not very well understood on classes of CNF-formulas that are known to be tractable. It is known that formulas whose {\em primal graph} -- the graph whose vertices are the variables and two variables are linked by an edge if they appear together in at least one clause -- is of treewidth $k$ can be compiled into decDNNF of size $2^{\Omega(k)} p(|F|)$~\cite{OztokD14} where $p$ is a polynomial which does not depend on $k$. Thus, for reasonable values of $k$, if the right elimination order and the right caching policy is chosen, exhaustive DPLL can solve these instances efficiently. A similar result have been shown for $\beta$-acyclic CNF-formulas~\cite{Capelli17}. However, for other structural restrictions, nothing is known. In this paper, we investigate the complexity of exhaustive DPLL on instances of bounded {\em incidence treewidth}, that is, instances whose treewidth of the incidence graph is bounded. It is known~\cite{SamerS10} that these instances can be solved in FPT time, that is, more precisely, in time $2^{\Omega(k)}p(|F|)$ where $k$ is the incidence treewidth and $p$ is a polynomial that does not depend on $k$. It is not known however whether these instances can be compiled into FPT-size decDNNF. We  answer this question on two natural restrictions of decDNNF: we show that they cannot represent instances of bounded incidence treewidth efficiently. The restrictions of decDNNF that we are interested in correspond to natural restriction on exhaustive DPLL. The first restriction we are interested in is a restriction known as {\em structuredness}, which has been introduced in~\cite{PipatsrisawatD08}. This restriction corresponds to the trace of the tool {\tt c2d}~\cite{OztokD14}. The other restriction we will be interested in is a restriction that we call $\adobdd$. Intuitively, it corresponds to a run of exhaustive DPLL when the elimination order of the variables is fixed at the beginning of the algorithm. For both restrictions and for every $k$, we show a  $|F|^{\Omega(k)}$ lower bound for an infinite family of instances of incidence treewidth $k$, showing that these restrictions cannot represent bounded incidence treewidth instances efficiently. Proving such lower bound on unrestricted decDNNF is still open.

It is worth noting that both of the considered restrictions of decDNNF are FPT on CNFs of bounded primal treewidth. Indeed, the FPT result for structural decDNNFs follows from Section 5 of~\cite{OztokD14}. In fact the decDNNF constructed there can be shown to be an $\wedge_d$-OBDD as well% \footnote{Partially order the variables according to the v-tree used for the construction of the decDNNF. In particular $u<v$ if the parent of the leaf labelled with $u$ is an ancestor of the parent of the leaf labelled with $v$. Then arbitrarily extend this partial order into a linear   one. Note that if $u<v$ is determined by the arbitrary extension then $u$ and $v$ will always be `separated' by a decomposable AND and hence this ordering will be `vacuously' respected by the resulting decDNNF.}
. In light of the above, we believe that our results are significant from the knowledge compilation perspective because they are the first to show \emph{different} parameterized complexities for primal
and incidence treewidth.

The paper is organised as follows. Section~\ref{sec:def} contains the definitions of the different notions that are needed to understand the results of this paper. Section~\ref{sec:sdec} is dedicated to the proof of the lower bound on structured decDNNF and Section~\ref{sec:aobdd} is dedicated to the proof of the lower bound on $\adobdd$. 

% \begin{itemize}
% \item Practical $\sSAT$ solvers compiles into decDNNF
% \item Structural approach: FPT algorithm for bounded incidence treewidth
% \item Limitation of these approaches for structural restriction of decDNNF with instances ``almost'' easy:
% \begin{itemize}
% \item structured decDNNF: c2d by Darwiche,
% \item $\adobdd$: introduced formally here, corresponds to preprocessing the elimination order. Implicitly used for primal treewidth and $\beta$-acyclic formulas (my LICS paper).
% \end{itemize}
% \end{itemize}

% \subsection{Previous work}
% \label{sec:previouswork}

% \begin{itemize}
% \item Compilation of ptw to decDNNF.
% \item Compilation of itw (and more) to d-DNNF, and SDD recently.
% \item Non-FPT lower bounds of Igor.
% \item General lower bounds on DNNF.
% \end{itemize}

\section{Preliminaries}
\label{sec:def}

\subsection{Boolean functions}
\label{sec:boolfun}

\paragraph{Assignments and Boolean functions.} Let $X$ be a finite set of variables. A {\em truth assignment} on $X$ is a mapping from $X$ to $\{0,1\}$. The set of truth assignments on $X$ is denoted by $\{0,1\}^X$. Given $Y \subseteq X$ and $\tau \in \{0,1\}^X$, we denote by $\tau|_Y$ the restriction of $\tau$ on $Y$. Let $X,X'$ be two sets of variables, $\tau \in \{0,1\}^X$ and $\tau' \in \{0,1\}^{X'}$. We denote by $\tau \simeq \tau'$ if $\tau|_{X \cap X'} = \tau'|_{X \cap X'}$. Moreover, if $\tau \simeq \tau'$, we let $\tau \cup \tau'$ be the truth assignment on variables $X \cup X'$ such that $(\tau \cup \tau')|_X = \tau$ and $(\tau \cup \tau')|_{X'} = \tau'$. A {\em boolean function} on variables $X$ is a mapping from $\{0,1\}^X$ to $\{0,1\}$. Given $\tau \in \{0,1\}^X$, we denote by $\tau \models f$ iff $f(\tau) = 1$. For $Y \subseteq X$ and $\tau \in \{0,1\}^Y$, we denote by $f[\tau]$ the boolean function on variables $X \setminus Y$ such that for every $\tau' \in \{0,1\}^{X \setminus Y}$, $f[\tau](\tau') = f(\tau \cup \tau')$. 

\paragraph{CNF-formulas.} A {\em literal} over $X$ is either a variable $x$ or its negation $\neg x$ for $x \in X$. For a literal $\ell$, we denote by $\var(\ell)$ the underlying variable of $\ell$. A {\em clause} $C$ on $X$ is a set of literals on $X$ such that for every $\ell, \ell' \in C$, if $\var(\ell) = \var(\ell')$ then $\ell = \ell'$. We denote by $\var(C) = \{\var(\ell) \mid \ell \in C\}$. A CNF-formula $F$ on $X$ is set of clauses on $X$. We denote by $\var(F) = \bigcup_{C \in F} \var(C)$. A CNF-formula is {\em monotone} if it does not contain negated literals. A CNF-formula naturally defines a Boolean function as follows: given a truth assignment $\tau$, we naturally extend a truth assignment on $X$ to literals on $X$ by defining $\tau(\neg x) = 1-\tau(x)$. We say that a clause $C$ on $X$ is satisfied by $\tau$ if there exists $\ell \in C$  such that $\tau(\ell) = 1$. A CNF-formula $F$ on $X$ is satisfied by $\tau$ if for every $C \in F$, $C$ is satisfied by $\tau$. The Boolean function defined by a CNF-formula $F$ is the function mapping satisfying assignments to $1$ and the others to $0$.

\subsection{Graphs}
\label{sec:graphs}

We assume the reader familiar with the basics notions of graph theory. For an introduction to this topic, we refer to~\cite{Diestel05}. 

Given a graph $G = (V,E)$ and $x \in V$, we denote by $G\setminus x$ the graph having vertices $V \setminus \{x\}$ and edges $E \setminus \{e  \in E \mid x \in e\}$, that is, $G \setminus x$ is obtained by removing $x$ and all its adjacent edges from $G$.

\paragraph{Matchings.} A {\em matching} is a set $M \subseteq E$ of edges such that for every $e,f \in M$, if $e\neq f$ then $e \cap f = \emptyset$. A matching $M$ in an {\em induced matching} if for every $\{u,v\}, \{u',v'\} \in M$, we have  $\{u,u'\} \notin E$,  $\{v,v'\} \notin E$,  $\{u,v'\} \notin E$ and  $\{u',v\} \notin E$. 

In bounded degree graphs, we can always extract a large induced matching from a matching:
\begin{lemma}
  \label{lem:extractinduced} Let $G = (V,E)$ be a graph of degree $d$ and $M \subseteq E$ be a matching of $G$. There exists an induced matching $M' \subseteq M$ of size at least $|M|/2d$.
\end{lemma}
\begin{proof}
  Take an edge $\{u,v\} \in M$ and remove from $M$ every edges $\{u',v'\} \in M$ such that either $\{u,u'\} \in E$,  $\{v,v'\} \in E$,  $\{u,v'\} \in E$ or $\{u',v\} \in E$. Since $G$ is of degree $d$, there are at most $2d$ such edges ($d$ adjacent to $u$, $d$ adjacent to $v$). Thus, we remove at most $2d$ edges. We can repeat the process until we have an induced matching.  
\end{proof}

\paragraph{Treewidth.} Treewidth is a graph parameter that intuitively measures the distance from a graph to a tree.  Given a graph $G = (V,E)$, a tree decomposition of $G$ is a tree $T = (V_T,E_T)$ where each $t \in V_t$ is labelled with $B_t \subseteq V$, called a {\em bag}, such that:
\begin{itemize}
\item for every $e \in E$, there exists $t \in V_T$ such that $e \subseteq B_t$,
\item for every $x \in V$, $\{t \mid x \in B_t\}$ is a connected subtree of $T$.
\end{itemize}
The width of a tree decomposition is $\max_{t \in V_T} (|B_t|-1)$. The tree width $k$ of $G$ is the smallest $k$ such that there exists a tree decomposition of $G$ of width $k$. Given a graph $G$, we denote its treewidth by $\tw(G)$.

\paragraph{Graphs for CNF-formulas.} The structure of a CNF-formula is usually studied using graphs representing the interaction between the clauses and the variables. Given a CNF-formula $F$, we define two graphs characterising its structure:

\begin{itemize}
\item {\em The primal graph of $F$} is the graph having vertices $\var(F)$ and such that $\{x,y\}$ is an edge if and only if there exists a clause $C \in F$ such that $\{x,y\} \subseteq C$.  
\item {\em The incidence graph of $F$} is the bipartite graph having vertices $\var(F) \cup F$ and such that for $x \in \var(F)$ and $C \in F$, $\{x,C\}$ is an edge if and only if $x \in \var(C)$.
\end{itemize}

The {\em primal treewidth} of a CNF-formula $F$, denoted by $\ptw(F)$, is the treewidth of the primal graph of $F$. The  {\em incidence treewidth} of a CNF-formula $F$, denoted by $\itw(F)$, is the treewidth of the incidence graph of $F$. Both measures are related thanks to the following theorem:
\begin{theorem}[\cite{Szeider04}]
  \label{thm:ptwvsitw} For every CNF-formula $F$, we have $\itw(F) \leq \ptw(F)$.
\end{theorem}

\subsection{Decision DNNF}
\label{sec:decdnnf}

In this section, we define Decision Decomposable Negation Normal Form (decDNNF for short). These circuits have first been introduced by Darwiche~\cite{HuangD05} as special cases of (deterministic) DNNF~\cite{Darwich01,Darwiche01MC}. Decision DNNF can also be seen as branching programs (FBDD), augmented with decomposable $\land$-nodes~\cite{BeameLRS13}. To make the proofs and definition easier to understand, we choose in this paper to present decDNNF as a generalisation of branching programs. 

\paragraph{$\afbdd$.} Let $X$ be a finite set of variables. An $\afbdd$ $Z$ on variables $X$ is a DAG with one distinguished node called the {\em root}. The nodes of $Z$ without outgoing edges are called the {\em sinks} and are labeled with a constant $0$ or $1$. The internal nodes of $G$ can be of two types:
\begin{itemize}
\item the {\em decision nodes}, labelled with a variables $x \in X$ and having two distinguished outgoing edges: one is labelled by $1$ (represented with a solid line in our figures) and the second by $0$ (represented with a dashed line in our figures), 
\item the {\em $\land$-nodes}, labelled with $\land$ and having two unlabeled outgoing edges.
\end{itemize}
Moreover, $Z$ respects the following condition: if there exists a directed path in $Z$ from a decision node $\alpha$ to a decision node $\beta$ then $\alpha$ and $\beta$ are labelled by a different variables.

Let $Y \subseteq X$ and  $\tau \in \{0,1\}^Y$.  A directed path $P$ in $Z$ is said to be {\em compatible with $\tau$} if for every edge $(\beta,\gamma)$ of $P$ where $\beta$ is a decision node on variable $v$, we have: 
\begin{itemize}
\item $v \in Y$ and
\item $(\beta,\gamma)$ is labelled by $\tau(v)$.
\end{itemize}
We say that a node $\beta$ is {\em reached by $\tau$} if there exists a path from the root of $Z$ to $\beta$ that is compatible with $\tau$. For $\tau \in \{0,1\}^X$, we say that $\tau$ satisfies $Z$ if no $0$-sink is reached by $\tau$. The Boolean function computed by $Z$ is the Boolean function that maps satisfying assignment of $Z$ to $1$ and the others to $0$. Figure~\ref{fig:decdnnfexample} pictures an $\land$-FBDD computing the Boolean Function on variables $x,y,z$ that is true if and only if exactly two variables among $x,y$ and $z$ are set to $1$. The nodes reached by $\tau = \{x\mapsto 0, y\mapsto 1, z \mapsto 1\}$ are depicted in bold font. It is a satisfying assignment as only the sink $1$ is reached.

\begin{figure}
  \centering
  \begin{tikzpicture}
    \input{decdnnf.tikz}
  \end{tikzpicture}
  \caption{A decDNNF computing the assignment $\tau$ of $\{x,y,z\}$ such that $\tau(x) + \tau(y)+ \tau(z) = 2$. In bold, the nodes reached by $\{x\mapsto 0, y\mapsto 1, z \mapsto 1\}$.}
  \label{fig:decdnnfexample}
\end{figure}
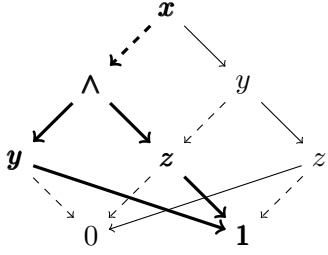

The {\em size} of $Z$, denoted by $|Z|$, is defined to be the number of edges of the underlying DAG of $Z$. Given a $\afbdd$ $Z$ on variables $X$, we denote by $\var(Z)$ the set of variables labelling its decision nodes. Observe that $\var(Z) \subseteq X$ but that the inclusion may be strict. Given a node $\alpha$ of $Z$, we denote by $Z_\alpha$ the $\afbdd$ whose root is $\alpha$ and containing every node that can be reached from $\alpha$ in $Z$.

An FBDD is an $\afbdd$ that has no $\land$-nodes.

\paragraph{Decision DNNF.} Let $\alpha$ be an $\land$-node and let $\alpha_1,\alpha_2$ be its two successor. The node $\alpha$ is said to be {\em decomposable} if $\var(Z_{\alpha_1}) \cap \var(Z_{\alpha_2}) = \emptyset$. In this paper, we sometimes use the notation $\land_d$ to stress out the fact that A decision DNNF, decDNNF for short, is an $\afbdd$ whose all $\land$-nodes are decomposable. It is easy to check that the $\afbdd$ depicted in Figure~\ref{fig:decdnnfexample} is an decDNNF since the variables on the left side of the $\land$-nodes are $\{y\}$ and the variables on the right side are $\{z\}$.

\paragraph{Normalising decDNNF.} Observe that if $\alpha$ is a decomposable $\land$-node with successors $\alpha_1,\alpha_2$ then $Z_{\alpha_1}$ and $Z_{\alpha_2}$ do not share any decision nodes otherwise they would not have disjoint variables. That is, the only nodes they share do not contain variables and then computes constants that can be replaced by sinks. Thus, we can assume without any increase in the size of $Z$, that the only nodes that are shared by $Z_{\alpha_1}$ and $Z_{\alpha_2}$ are sinks. By copying each sink such that they have exactly one ingoing edge, we can construct a decDNNF equivalent to $Z$ of size at most $2|Z|$ such that for every $\land$-node with successors $\alpha_1,\alpha_2$, $Z_{\alpha_1}$ and $Z_{\alpha_2}$ are disjoint circuits. In this paper, we will always assume that every decomposable $\land$-nodes have this property.

\subsection{Restricted decDNNF}
\label{sec:restdecdnnf}

In the framework of branching programs, FBDD are distinguished from OBDD by the fact that in an OBDD, each source-sink path has to test the variables in the same order. This ordering constraint allows to have more tractable transformations on OBDD than on FBDD. In this section, we introduce two generalisations of such an idea for decDNNF. 

\paragraph{Structured decDNNF.} Structured DNNF have first been introduced in~\cite{PipatsrisawatD08} as a restriction of DNNF. The structuredness restriction corresponds to a restriction on the way variables are partitioned the nodes of the DNNF.

Let $X$ be a finite set of variables. A vtree $T$ on variables $X$ is a rooted tree whose leaves are in one to one correspondence with $X$ and such that every non-leaf node $t$ of $T$ has exactly two children. Given a node $t$ of $T$, we denote by $T_t$ the subtree of $T$ rooted in $t$ and $\var(t)$ the variables labelling the leaves of $T_t$.

A decDNNF $Z$ on variables $X$ {\em respects} a vtree on variables $X$ if: for every $\land$-node of $Z$ having children $v_1,v_2$, there exists a node $t$ in $T$ having children $t_1, t_2$ such that $\var(Z_{v_1}) \subseteq \var(t_1)$ and $\var(Z_{v_2}) \subseteq \var(t_2)$. For every decision node in $Z$ having children $v_1,v_2$ and testing the variable $x$, there exists $t$ in $T$ having children $t_1,t_2$ such that $x \in \var(t_1)$ and $\var(Z_{v_1}) \cup \var(Z_{v_2}) \subseteq \var(t_2)$. 

The class of decDNNF respecting a given vtree $T$ is denoted by $\text{decDNNF}_T$. A decDNNF $Z$ is said to be {\em structured} if there exists a vtree $T$ such that $Z$ respects $T$. 

\paragraph{Decomposable $\aobdd$ ($\adobdd$).} OBDDs are a well-known restriction of FBDDs where the variables have to appear on a fixed order along a path from the root to the sinks. We naturally generalise this restriction to $\land$-FBDD as follows: let $Z$ be an $\afbdd$ on variables $X$ and $<$ a linear order on $X$. We say that $Z$ respects the order $<$ if for every decision nodes $\alpha$ labeled with $x$ and $\beta$ labeled with $y$ of $Z$, if there exists a directed path from $\alpha$ to $\beta$ in $Z$, then $x < y$. 

The class of $\land$-FBDD respecting order $<$ is denoted by $\aobdd_<$. The class of $\land$-FBDD respecting some order $<$ on its variables is denoted by $\land$-OBDD. A decomposable $\aobdd$ is an $\aobdd$ whose $\land$-nodes are all decomposable. We denote the class of decomposable $\aobdd$ by $\adobdd$. An OBDD is an $\aobdd$ that has no $\land$-nodes.

\paragraph{Comparing structured decDNNF and $\adobdd$.} It is easy to see that the decDNNF represented in Figure~\ref{fig:decdnnfexample} is also a $\adobdd$ for the order $x < y < z$. Moreover, it is also structured since it respects the vtree given in Figure~\ref{fig:vtree1}.

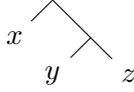
\begin{figure}
  \centering
  \begin{tikzpicture}[scale=.5]
    \node (x) at (-1,-1) {$x$};
    \node (y) at (0,-2) {$y$};
    \node (z) at (2,-2) {$z$};
    \draw (x) -- (0,0) -- (1,-1);
    \draw (y) -- (1,-1) -- (z);
  \end{tikzpicture}
  \caption{A vtree respected by the decDNNF of Figure~\ref{fig:decdnnfexample}}
  \label{fig:vtree1}
\end{figure}

We will see that we can use the instances of Section~\ref{sec:sdec} to separate structured decDNNF from $\adobdd$, by construction instances having structured decDNNF of size at least $n^{\Omega(\log n)}$ and polynomial size $\adobdd$. A stronger separation of both classes may be found in~\cite{Capelli17} where CNF-formulas with $n$ variables having polynomial size $\adobdd$ are proven to have no structured decDNNF of size $2^{\Omega(\sqrt(n))}$.

However, it is not clear to us if structured decDNNF are weaker than $\adobdd$. Figure~\ref{fig:sdec} gives a decDNNF which is structured -- it indeed respects the vtree of Figure~\ref{fig:vtree1} -- but is not an $\adobdd$. Even if this example is easy to transform into an $\adobdd$, we do not know if such a transformation always exists.

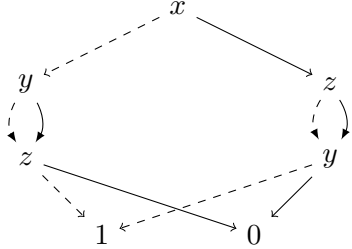
\begin{figure}
  \centering
  \begin{tikzpicture}
    \input{sdec.tikz}
  \end{tikzpicture}
  \caption{A structured decDNNF respecting the vtree of Figure~\ref{fig:vtree1} that are not an $\adobdd$.}
  \label{fig:sdec}
\end{figure}

\section{Lower bound for structured decDNNF}
\label{sec:sdec}

This section is dedicated to the proof of our first main theorem:
\begin{theorem}
\label{thm:sdeclb}
  For every $k$, there exists an infinite family $\calF$ of CNF of incidence treewidth $k$ having no structured decDNNF smaller than $n^{\Omega(k)}$ where $n = |\var(F)|$.
\end{theorem}

The instances we use for proving this lower bounds are the following: for a graph $G = (V,E)$, we define a CNF $F_G$ from $G$ on variables $V$ as follows:
\[ F_G = (\bigvee_{x \in V} \neg x) \land \bigwedge_{(x,y) \in E} (x \lor y). \]
The clause $\bigvee_{x \in V} \neg x$ is denoted by $C_G$. 

The proof then goes as follows. After some preliminaries observations on structured decDNNF in Section~\ref{sec:obssdecdnnf}, we show in Section~\ref{sec:sdecinstances} that the smallest structured decDNNF computing $F_G$ is actually an FBDD. We finally show the lower bound in Section~\ref{sec:sdeclb} by using a non-FPT lower bound from~\cite{Razgon14} on the size of FBDD computing $F_G \setminus C_G$,

\subsection{Preliminary observations on structured decDNNF}
\label{sec:obssdecdnnf}

In this section, we are interested in structured decDNNF having a sequence of decision nodes testing all of its variables. We show that these structured decDNNF can be efficiently simulated by FBDDs. Given a decDNNF $Z$ on variable $X$ and $Y \subseteq X$, we call a path $P = (v_1,\ldots,v_{n})$ of $Z$ a decision path on $Y$ if $P$ verifies the following:
\begin{itemize}
\item for all $i$, $v_i$ is a decision node,
\item for all $x \in Y$, there exists $i$ such that $v_i$ is labelled with the variable $x$.
\end{itemize}

We say that a vtree $T$ on $X$ is {\em linear} if for every non-leaf node of $T$, at least one of its child is a leaf (see Figure~\ref{fig:pathlike}). 

\begin{figure}
  \centering
  \begin{tikzpicture}[scale=.5]
    \def\mylist{0,...,3} 
  \foreach \i in \mylist { 
    
    \pgfmathtruncatemacro{\k}{\i-1};
    \pgfmathtruncatemacro{\j}{-\i-1};
    \node (x\i) at (\k,\j) {$x_\i$}; 
    \draw (\i,-\i) -- (x\i);
  }
  \foreach \i in {0,...,2} { 
    \pgfmathtruncatemacro{\j}{\i+1};
    \draw (\i,-\i) -- (\j,-\j); 
  }
  \node (x4) at (4,-4) {$x_4$};
  \draw (x4) -- (3,-3);
  \end{tikzpicture}
  \caption{A linear vtree.}
  \label{fig:pathlike}
\end{figure}
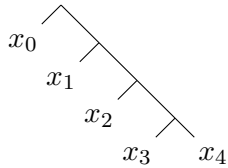

It has been observed but not proven by Darwiche and Pipatsrisawat in~\cite{PipatsrisawatD08} that a decDNNF respecting a linear vtree can be efficiently simulated with an OBDD. Lemma~\ref{lem:pathlikeOBDD} gives a proof of a slightly less general version of this observation: we prove that such decDNNF can be efficiently simulated by FBDD. It is not hard to show from there that the FBDD constructed in the proof can be easily turned into an OBDD.
\begin{lemma}
\label{lem:pathlikeOBDD}
  Let $T$ be a linear vtree on variables $X$ and $Z$ a decDNNF respecting $T$. There exists an FBDD $Z'$ of size $O(|Z|)$ equivalent to $Z$.
\end{lemma}
\begin{proof}[Proof.] Let $Z$ be a decDNNF respecting a linear vtree $T$. We explain how to remove $\land$-nodes from $Z$. Let $v$ be an $\land$-node of $Z$ having children $v_1,v_2$ and let $t$ be the node in $T$ with children $t_1,t_2$ such that $\var(Z_{v_1}) \subseteq \var(t_1)$ and $\var(Z_{v_2}) \subseteq \var(t_2)$. Since $T$ is linear, either $t_1$ or $t_2$ is a leaf of $T$. Assume wlog that $t_1$ is a leaf labelled with $x \in X$. Thus $\var(Z_{v_1}) \subseteq \{x\}$. In other words, $Z_{v_1}$ computes either $x \mapsto 0$, $x \mapsto 1$, $x \mapsto x$ or $x \mapsto \neg x$. If $Z_{v_1}$ computes $x \mapsto 0$, then $Z_v$ computes the function $0$. Hence, we can replace $v$ by a $0$-sink. If $Z_{v_1}$ computes $x \mapsto 1$, then $Z_v$ computes the same function as $Z_{v_2}$, hence we can remove $v$ and $v_1$ from $Z$ and connects every predecessor of $v$ directly to $v_2$.  If $Z_{v_1}$ computes $x \mapsto x$, then $Z_v$ computes $x \land Z_{v_2}$. Hence, we can replace $v$ by a decision-node whose $1$-edge is connected to $v_2$ and $0$-edge is connected to the constant $0$. The case where $Z_{v_1}$ computes $x \mapsto \neg x$ is symmetric. 

It is easy to check that for each case, the transformation does not change the function computed by $Z$ nor does it increase its size. Applying this transformation to every $\land$-node of $Z$ leads to an FBDD $Z'$, equivalent to $Z$ and of size at most $|Z|$.
\end{proof}

The following lemma suggests that long paths in structured decDNNF put strong constraints on the respected vtree:
\begin{lemma}
  \label{lem:longpathPL} Let $Z$ be a structured decDNNF on variables $X$ such that there exists a source-sink decision path on $X$. Every vtree $T$ such that $Z$ respects $T$ is linear. 
\end{lemma}
\begin{proof}
  The proof is by induction on $n = |X|$. If $n \leq 2$, then every vtree is linear so there is nothing to prove. Now assume that the result holds for any structured decDNNF $Z$ with $n$ variables and let $X$ be such that $|X|=n+1$ and $Z$ be a structured decDNNF such that there exists a path $v_1,\ldots,v_{n+1},v_{n+2}$ in $Z$ as in the statement of the lemma. Let $T$ be a vtree such that $Z$ respects $T$. By definition, $v_1$ is a decision node on $x_1$ and one of its children is $v_2$ and we have $\var(Z_{v_2}) = X \setminus \{x_1\}$ since every variable of $X$ is tested on the path. 

Since $Z$ is structured, there exists a node $t$ in $T$ with children $t_1,t_2$ such that $x_1 \in \var(t_1)$ and $X \setminus \{x_1\} = \var(Z_{v_2}) \subseteq \var(t_2)$. Thus, $\var(t_1) = \{x_1\}$ and $X \setminus \{x_1\} = \var(t_2)$. That is, $t$ is the root of $t$ and $t_1$ is a leaf labelled with $x_1$. 

Now, observe that $Z_{v_2}$ is a structured decDNNF on variables $X \setminus \{x_1\}$. Moreover, $Z_{v_2}$ respects $T_{t_2}$ and there exists a path $v_2,\ldots,v_{n+1},v_{n+2}$ of length $n$ of decision nodes and such that for every $x \in X\setminus\{x_1\}$, there exists $i$ such that $v_i$ tests $x$. Thus, we can apply the induction hypothesis on $Z_{v_2}$ and it follows that $T_{t_2}$ is linear. 

Thus let $t$ be a non-leaf node of $T$. If $t$ is the root, then one of its child is a leaf labelled by $x_1$. Otherwise, $t$ is a node of $T_{t_2}$. Since $T_{t_2}$ is linear, there exists a child of $t$ which is a leaf labelled with a variable $x \in X \setminus \{x_1\}$. Thus, $T$ is linear.
\end{proof}

An immediate corollary of Lemma~\ref{lem:longpathPL} and Lemma~\ref{lem:pathlikeOBDD} is the following:
\begin{corollary}
    \label{cor:longpathOBDD} Let $Z$ be a structured decDNNF on variables $X$ such that there exists a source-sink decision path on $X$ $v_1,\ldots,v_n,v_{n+1}$. There exists an FBDD of size $O(|Z|)$ computing the same function as $Z$.
\end{corollary}

\subsection{Instances}
\label{sec:sdecinstances}

In this section, we construct CNFs such that every structured decDNNF computing them must contain a path of decision-nodes testing every variables. By Corollary~\ref{cor:longpathOBDD}, it implies that every structured decDNNF computing them can be efficiently simulated by FBDD. We then use lower bounds from~\cite{Razgon14} on FBDD to conclude.

The two following lemmas show that the clause $C_G$ forces a long path of decision nodes in any decDNNF computing $F_G$.

\begin{lemma}
  \label{lem:fgminusx} Let $G = (V,E)$ be a graph, $|V| \geq 2$ and $x \in V$. We have
\[ F_G[x \mapsto 1] = F_{G \setminus x}. \]
\end{lemma}
\begin{proof}
We show that both functions have the same satisfying assignments.  Let $\tau \models F_{G \setminus x}$. We claim that $\tau' = \tau \cup \{x \mapsto 1\} \models F_G$. Indeed, since $\tau \models F_{G \setminus x}$ and $|V| \geq 2$, we have that $\tau \models \bigvee_{y \in V \setminus \{x\}} \neg y$ thus there exists $y \in V \setminus \{x\}$ such that $\tau(y) = 0$  and then $\tau \models C_G$. Now, let $(u,v) \in E$. If both $u \neq x$ and $v \neq x$ then $(u,v)$ is also an edge of $G \setminus x$ and then $\tau \models u \vee v$. Now if $u = x$, then clearly $\tau' \models u \vee v$.  Thus, $\tau' \models F_G$.

Now let $\tau'$ be such that $\tau' \models F_G$ and $\tau'(x) = 1$. Let $\tau = \tau'|_{V \setminus \{x\}}$. We claim that $\tau \models F_{G \setminus x}$. Indeed, let $(u,v)$ be an edge of $G \setminus x$. By definition of $G \setminus x$, $(u,v)$ is an edge of $G$ with $u \neq x$ and $v \neq x$. Thus since $\tau' \models u \vee v$ we have $\tau \models u \vee v$. Moreover, since $\tau' \models C_G$ and $\tau'(x) = 1$, there exists $y \in V\setminus \{x\}$ such that $\tau'(y) = 0$. Thus $\tau' \models \bigvee_{z \in V \setminus \{x\}} \neg z$.
\end{proof}

\begin{lemma}
  \label{lem:fglongpath} Let $G=(V,E)$ be a graph and let $Z$ be a decDNNF computing $F_G$. There exists a source-sink decision path on $X$ $v_1,\ldots,v_n,v_{n+1}$ in $Z$.
\end{lemma}
\begin{proof}
  The proof is by induction on $n = |V|$. If $n = 1$, then $V = \{x\}$ and $F_G = \neg x$. Thus $F_G$ is reduced to a decision node on $x$ and the result is trivial. 

  Let $n \geq 1$, $G = (V,E)$ be a graph with $|V| = n+1 \geq 2$ and assume that the result holds for any graph having $n$ nodes. Let $Z$ be a decDNNF for $F_G$. We start by proving that we can assume that the source of $Z$ is necessarily a decision-node. Indeed, assume that the source of $Z$ is a decomposable $\land$-node with children $v_1,v_2$. Recall that we have shown that w.l.o.g, we can assume that $\var(Z_{v_1}) \neq \emptyset$ and $\var(Z_{v_1}) \neq \emptyset$. It means that $F_G$ can be rewritten as $F_1 \land F_2$ with $\var(F_1) = V_1 \neq \emptyset$ and $\var(F_2) = V_2 \neq \emptyset$. Let $\tau_1 \in \{0,1\}^{V_1}$ be the assignment such that $\tau_1(x) = 1$ for every $x \in V_1$ and let $\tau_2 \in \{0,1\}^{V_2}$ be the assignment such that $\tau_2(x) = 1$ for every $x \in V_2$. 

Assume first that $\tau_1 \not \models F_1$. In this case, choose $y \in V_2$ and define $\tau \in \{0,1\}^V$ as $\tau(y) = 0$ and for every $x \in V \setminus \{y\}$, $\tau(x) = 1$. It is easy to see that $\tau \models F_G$ since $C_G$ is satisfied as $\tau(y) = 0$ and if $(u,v) \in E$, then w.l.o.g we have $u \neq y$. Therefore $u \lor v$ is satisfied since $\tau(u) = 1$. However, $\tau \not \models F_1$ since $\tau|_{V_1} = \tau_1$, thus $\tau \not \models Z$ which contradicts the fact that $Z$ computes $F$. If $\tau_2 \not \models F_2$, we have a similar  contradiction. 

Finally, assume that $\tau_1 \models F_1$ and $\tau_2 \models F_2$. Thus, the assignment $\tau = \tau_1 \cup \tau_2 \models F_1 \land F_2 = F$. However, $\tau$ does not satisfy $C_G$, contradictions. 

Thus, the source $s$ of $Z$ is a decision-node. Let $x$ be the variable tested by $s$ of $Z$ and let $v$ be the child of $v$ that corresponds to $x \mapsto 1$. We have that $Z_v$ computes $F_G[x \mapsto 1]$. By Lemma~\ref{lem:fgminusx}, since $|V| \geq 2$, $F_G[x \mapsto 1] = F_{G \setminus x}$. By induction, there exists a source-sink path of decision nodes in $Z_v$ testing every variable in $V \setminus \{x\}$. Thus we have a source-sink path of decision nodes in $Z$ testing every variable in $V$.
\end{proof} 

By combining Corollary~\ref{cor:longpathOBDD} and Lemma~\ref{lem:fglongpath}, we have the following:
\begin{corollary}
\label{cor:fgobdd} Let $G = (V,E)$ be a graph and let $Z$ be a structured decDNNF computing $F_G$. There exists an FBDD of size at most $|Z|$ computing $F_G$.
\end{corollary}

\subsection{Lower bound}
\label{sec:sdeclb}

The following has been proven in~\cite{Razgon13}:
\begin{theorem}[\cite{Razgon13}]
\label{thm:razgon}
There exists a constant $c$ such that for every $k \in \mathbb{N}$, there exists an infinite family $\calG_k$ of graphs of treewidth at most $k$  such for every graph $G = (V,E) \in \calG_k$, any FBDD computing $F_G \setminus C_G$ is of size at least $|V|^{ck}$.
\end{theorem}

We will Theorem~\ref{thm:razgon} to prove our lower bound. We start by showing that $F_G$ is not easier to compute than $F_G \setminus C_G$ with FBDD:
\begin{lemma}
\label{lem:fgfromraz}
  Let $G = (V,E)$ be a graph and $Z$ be an FBDD computing $F_G$. Then, there exists an FBDD $Z'$ of size at most $|V| \cdot |Z|$ computing $F_G \setminus C_G$.
\end{lemma}
\begin{proof}[Proof.]
  Let $(x_1,\ldots,x_n)$ be an arbitrary ordering of $V$. We let $\tau_i : \{x_1, \ldots, x_i\} \rightarrow \{0,1\}$ be defined as $\tau_i(x_j) = 1$ for $j < i$ and $\tau_i(x_i) = 0$. 

We create an FBDD $Z'$ as follows: $Z'$ has $n$ decision nodes $\alpha_i$, $\alpha_i$ testing variable $x_i$ and $n$ disjoint FBDD, each of them computing $Z[\tau_i]$, which can be done with an FBDD of size at most $|Z|$ by projecting $Z$ on $\tau_i$. The root of $Z'$ is $\alpha_1$. 

For $i<n$, the outgoing edge of $\alpha_i$ that is labelled by $1$ goes into $\alpha_{i+1}$ and the outgoing edge of labelled by $0$ goes into the FBDD computing $Z[\tau_i]$. The outgoing edge of $\alpha_n$ labelled with $1$ goes into a $1$-source. 

It is easy to check that $|Z'| \leq |V| \cdot |Z|$ and that $Z'$ computes $F_G \setminus C_G$. Indeed, let $\tau$ be a satisfying assignment of $F_G \setminus C_G$. If $\tau$ assigns every variable to $1$. In this case, the corresponding path in $Z'$ goes through every $\alpha_i$ and ends up in a $1$-sink. Otherwise, let $x_i$ be the smallest variables that is set to $0$ by $\tau$. In this case, the path corresponding to $\tau$ goes through $\alpha_1, \ldots, \alpha_i$ and then into the FBDD computing $Z[\tau_i]$. However, it is readily verified that $\tau \simeq \tau_i$. Moreover, since $\tau \models C_G$, $Z[\tau_i]$ is satisfied by $\tau$ if and only if $\tau$ satisfies $F_G \setminus C_G$. 
\end{proof}

Moreover, the incidence treewidth of $F_G$ is directly related to the treewidth of $G$:
\begin{lemma}
\label{lem:itwfg}
Let $G = (V,E)$ be a graph of treewidth $k$. We have $\itw(F_G) \leq k+1$.
\end{lemma}
\begin{proof}[Proof (sketch).]
  It is easy to see that the primal graph of $F_G \setminus C_G$ is $G$ so the primal treewidth of $F_G \setminus C_G$ is $k$. Now, by Theorem~\ref{thm:ptwvsitw}, the incidence treewidth of $F_G \setminus C_G$ is at most $k$. Let $T$ be a tree decomposition for the incidence graph of $F_G \setminus C_G$. Construct a new tree decomposition by adding $C_G$ in every bag of $T$. It is easy to check that this is a tree decomposition of the incidence graph of $F_G$ of width $k+1$.
\end{proof}

\begin{proof}[Proof (of Theorem~\ref{thm:sdeclb}).]
  Let $\calG_k$ the family of graphs given by Theorem~\ref{thm:razgon} and let $\calF_k = \{F_G \mid G \in \calG_k \}$. By Lemma~\ref{lem:itwfg}, $\calF_k$ is an infinite family of formulas of incidence treewidth $k$. Now let $F \in \calF_k$ and let $G = (V,E) \in \calG_k$ the graph such that $F = F_G$. By Theorem~\ref{thm:razgon}, any FBDD computing $F \setminus C_G$ is of size at least $|V|^{ck}$. We conclude by using Lemma~\ref{lem:fgfromraz} that implies that the smallest FBDD computing $F_G$ is of size at least $|V|^{ck-1}$.
\end{proof}

\section{Lower bound for $\adobdd$}
\label{sec:aobdd}

This section is dedicated to the proof of our second lower bound:
\begin{theorem}
\label{thm:aobddlb}
  For every $k$, there exists an infinite family $\calF_k$ of CNF of incidence treewidth $k$ having no $\adobdd$ smaller than $n^{\Omega(k)}$ where $n = |\var(F)|$.
\end{theorem}

Unfortunately, contrary to the proof of Section~\ref{sec:sdec}, we cannot directly transfer a lower bound for FBDD. We have to generalise the technique of~\cite{Razgon13} to this class of circuits. Moreover, it can be shown that $F_G$ has $\adobdd$ of size FPT in the treewidth of $G$, thus we cannot use that same instances as in Section~\ref{sec:sdec}. Fortunately, we will see that only adding two long clauses is enough for proving such a lower bound.

We start by proving in Section~\ref{sec:propobdd} some general results concerning $\adobdd$. We then describe how we construct our instances in Section~\ref{sec:geninst}. Finally, Section~\ref{sec:inctw} gives a proof of Theorem~\ref{thm:aobddlb}.

\subsection{Properties of $\adobdd$}
\label{sec:propobdd}

Given a finite set $X$, $<$ an order on $X$ and $x \in X$, we denote by $[\leq x] = \{y \in X \mid y \leq x\}$. We similarly define $[<x]$, $[\geq x]$, $[>x]$. 

In the rest of this section, we fix $X$ a set of variables, $<$ an order on $X$, $Z$ an $\adobdd_<$ $Z$ on variables $X$, $u \in X$ and $\tau : [\leq u] \rightarrow \{0,1\}$. A node $\alpha$ of $Z$ reached by $\tau$ is {\em maximal} if $\alpha$ is the only node of $Z_\alpha$ that is reached by $\tau$. We will show that we can nicely characterise $Z[\tau]$ in terms of maximal nodes.

% \begin{lemma}
%   \label{lem:uniquepath} Let $\alpha$ be a gate reached by $\tau$. There exists a unique path $P_\alpha(\tau)$ from the root to $\alpha$ that is compatible with $\tau$.
% \end{lemma}
% \begin{proof}
%   Assume that this is not the case and let $P,Q$ be two paths from the root of $Z$ to $\alpha$ compatible with $\tau$. Since $P$ and $Q$ both start from the root of $Z$, there must exist a node $\gamma$ where they split, that is, $\gamma$ is a node with successor $\gamma_1$ and $\gamma_2$ such that $\gamma$ is in both $P$ and $Q$ but $\gamma_1$ is only in $Q$ and $\gamma_2$  on $Q$. We claim that this is not possible. Indeed, by definition, $\gamma$ is not a sink. Moreover $\gamma$ is not an $\land_d$-node. Indeed, $\alpha$ is in both $Z_{\gamma_1}$ and $Z_{\gamma_2}$ and if $\gamma$ were an $\land$-gate, $Z_{\gamma_1}$ and $Z_{\gamma_2}$ would be disjoint. Finally, assume that $\gamma$ is a decision node on variable $v$. Since $(\gamma,\gamma_1)$ is in $P$ and $P$ is compatible with $\tau$, we have $v \in [\leq u]$ and $(\gamma,\gamma_1)$ is labelled with $\tau(v)$. But since $(\gamma,\gamma_2)$ is in $Q$, $(\gamma,\gamma_2)$ is also labelled with $\tau(v)$ which contradicts the definition of decision nodes. Thus $\gamma$ does not exist, that is, $P = Q$.
% \end{proof}

\begin{lemma}
\label{lem:maxinode}
Let $\alpha, \beta$ be two nodes of $Z$. Let $P_\alpha$ and $P_\beta$ be two path from the root of $Z$ to $\alpha$ and $\beta$ respectively compatible with $\tau$.
\begin{itemize}
\item If $\beta$ is in $Z_\alpha$ then $P_\beta$ contains $\alpha$.
\item If $\beta$ is not in $Z_\alpha$ and $\alpha$ is not in $Z_\beta$ then $\var(Z_\alpha) \cap \var(Z_\beta) = \emptyset$.
\end{itemize}
\end{lemma}
\begin{proof}
First, assume that $\beta$ is in $Z_\alpha$ and that $\alpha$ is not in $P_\beta$. Since $P_\alpha$ and $P_\beta$ both start from the root, there must exist a node $\gamma$ where they split, that is, $\gamma$ is a node with successor $\gamma_1$ and $\gamma_2$ such that $\gamma$ is in both $P_\alpha$ and $P_\beta$ but $\gamma_1$ is only in $P_\alpha$ and $\gamma_2$  on $P_\beta$. We claim that this is not possible. Indeed, by definition, $\gamma$ is not a sink. Moreover $\gamma$ is not an $\land_d$-node. Indeed, $\beta$ is in both $Z_{\gamma_1}$ and $Z_{\gamma_2}$ and if $\gamma$ were an $\land$-node, $Z_{\gamma_1}$ and $Z_{\gamma_2}$ would be disjoint. Finally, assume that $\gamma$ is a decision node on variable $v$. Since $(\gamma,\gamma_1)$ is in $P_\alpha$ and $P_\alpha$ is compatible with $\tau$, we have $v \in [\leq u]$ and $(\gamma,\gamma_1)$ is labelled with $\tau(v)$. But since $(\gamma,\gamma_2)$ is in $P_\beta$, $(\gamma,\gamma_2)$ is also labelled with $\tau(v)$ which contradicts the definition of decision nodes. 

Now assume that $\beta$ is not in $Z_\alpha$ and $\alpha$ is not in $Z_\beta$. Thus $P_\alpha$ and $P_\beta$ split at a node $\gamma$ defined as in the previous paragraph. For the same reasons as before, $\gamma$ is not a decision node, nor a sink. Thus $\gamma$ is an $\land$-node, meaning that $\var(Z_{\gamma_1}) \cap \var(Z_{\gamma_2}) = \emptyset$. The result follows since $\var(Z_{\alpha}) \subseteq \var(Z_{\gamma_1})$ and $\var(Z_{\beta}) \subseteq \var(Z_{\gamma_2})$.
\end{proof}

It follows from Lemma~\ref{lem:maxinode}:
\begin{lemma}
\label{lem:reachable}
Let $W$ be the set of nodes $\alpha$ reached by $\tau$ that are maximal. We have: 
\[ Z[\tau] \equiv \bigwedge_{\alpha \in W} Z_\alpha. \]
And for every $\alpha,\beta \in W$, if $\alpha \neq \beta$ then $\var(Z_\alpha) \cap \var(Z_\beta) = \emptyset$.
\end{lemma}
\begin{proof}
  Let $\tau' : [>u] \rightarrow \{0,1\}$. We prove that $\tau' \models Z[\tau]$ if and only if $\tau' \models Z_\alpha$ for every $\alpha \in W$. First, assume that $\tau' \models Z[\tau]$, that is, $\tau \cup \tau' \models Z$. Let $\alpha \in W$ and $\sigma$ be a sink of $Z_\alpha$ reached by $\tau'$. By definition, $\alpha$ is reached by $\tau \cup \tau'$ in $Z$, thus, $\sigma$ is also reached by $\tau \cup \tau'$ in $Z$. Thus, since $\tau \cup \tau' \models Z$, $\sigma$ is labelled with $1$. 

Now assume that for every $\alpha \in W$, $\tau' \models Z_\alpha$. Let $\sigma$ be a sink reached by $\tau \cup \tau'$ in $Z$ by a path $P$. Let $\alpha$ be the last node on $P$ that is reached by $\tau$. We claim that $\alpha$ is maximal. Indeed, if $\alpha$ is a sink then it is clearly maximal. Otherwise, let $\alpha_1$ and $\alpha_2$ be the successor of $\alpha$ and assume that $\alpha_1$ is in $P$. If $\alpha$ is an $\land$-node, then $\alpha_1$ would be reached by $\tau$ contradicting the maximality of $\alpha$ on $P$. If $\alpha$ is a decision node for a variable $v \leq u$, then since $P$ is compatible with $\tau \cup \tau'$, $(\alpha,\alpha_1)$ is labelled with $(\tau \cup \tau')(v) = \tau(v)$, that is, $\alpha_1$ is reached by $\tau$, contradicting again the maximality of $\alpha$. Thus $\alpha$ is a decision node on a variable $v > u$. Now, assume there exists $\beta \neq \alpha$ in $Z_\alpha$ that is reached by $\tau$ and let $P_\beta$ be a path compatible with $\tau$ from the root of $Z$ to $\beta$. By Lemma~\ref{lem:maxinode}, $P_\beta$ contains $\alpha$ and then, it contains either $\alpha_1$ or $\alpha_2$. But since $\alpha$ is a decision node on $v > u$, $\tau$ does not assign $v$ and it contradicts the fact that $P_\beta$ is compatible with $\tau$. Thus $\alpha \in W$. By assumption. $\tau' \models Z_\alpha$, that is $\tau \cup \tau' \models Z$. 

Finally, let $\alpha,\beta$ be two distinct nodes of $W$. By definition of $W$, $\alpha$ is not in $Z_\beta$ and $\beta$ is not in $Z_\alpha$ and both are reached by $\tau$. By Lemma~\ref{lem:maxinode}, $\var(Z_\alpha) \cap \var(Z_\beta) = \emptyset$.
\end{proof}

Finally, we will use the following lemma:
\begin{lemma}
  \label{lem:rectangle} Let $Z$ be a $\adobdd_<$, $\tau : \var(Z) \rightarrow \{0,1\}$ such that $\tau \models Z$ and $\alpha$ a node reached by $\tau$. Then for every $\tau': \var(Z_\alpha) \rightarrow \{0,1\}$ such that $\tau' \models Z_\alpha$, we have $\tau'' := \tau|_{\var(Z) \setminus \var(Z_\alpha)} \cup \tau' \models Z$.
\end{lemma}
\begin{proof}
First of all, observe that there exists a path $P_\alpha$ from the root to $\alpha$ that is compatible with $\tau$. Moreover, every variable tested on this path have to be smaller than the variables in $Z_\alpha$ since $Z$ is an $\aobdd_<$. Thus, this path is also compatible with $\tau''$.

Now, assume toward a contradiction that there exists a path $P$ from the root of $Z$ to a $0$-sink that is compatible with $\tau''$. If $P$ contains $\alpha$ then it gives a path from $\alpha$ to a $0$-sink compatible with $\tau'$, contradicting the fact that $\tau' \models Z_\alpha$. Thus $P$ does not contain $\alpha$. Now let $\gamma$ with successor $\gamma_1,\gamma_2$ be the first node on which $P$ and $P_\alpha$ splits, that is, $\gamma$ is in $P$ and $P_\alpha(\tau'')$ but $\gamma_1$ is only in $P$ and $\gamma_2$ only in  $P_\alpha(\tau'')$. It is easy to see that $\gamma$ is not a decision node thus $\gamma$ is an $\land_d$-node. Thus, $\var(Z_{\gamma_1}) \cap \var(Z_{\gamma_2}) = \emptyset$, that is, $\var(Z_{\gamma_1}) \cap \var(Z_\alpha) = \emptyset$. Thus, $P$ is compatible with $\tau$, which contradicts the fact that $\tau \models Z$.
\end{proof}

\subsection{Generic instances}
\label{sec:geninst}

Given a graph $G = (V,E)$, we define $V^1 = \{v^1 \mid v \in V\}$, $V^2 = \{v^2 \mid v \in V\}$ and $F^2_G$ the formula whose variables are $V^1 \cup V^2$ and clauses are: 
\begin{itemize}
\item for every $\{u,v\} \in V$, $u^1 \vee v^2$ and $u^2 \vee v^1$,
\item $C_1 = \bigvee_{u \in V} \neg u^1$ and,
\item $C_2 = \bigvee_{u \in V} \neg u^2$.
\end{itemize}

\paragraph{Example.} Let $G$ be the triangle graph on vertices $\{x,y,z\}$. We have 
\[
\begin{aligned}
F^2_G = 
       & (x^1 \vee y^2) \land (x^2 \vee y^1) & \land \\ 
       & (z^1 \vee y^2) \land (z^2 \vee y^1) & \land \\ 
       & (x^1 \vee z^2) \land (x^2 \vee z^1) & \land \\
       & (\neg x^1 \vee \neg y^1 \vee \neg z^1) & \land \\
       & (\neg x^2 \vee \neg y^2 \vee \neg z^2). &
\end{aligned}
\]

\begin{lemma}
\label{lem:extractnode}
  Let $G = (V,E)$ be a graph, $<$ an order on $V^1 \cup V^2$ and $Z$ be an $\adobdd_<$ computing $F^2_G$. Let $u \in V^1 \cup V^2$ and $\tau : [\leq u] \rightarrow \{0,1\}$ and $W = \{v \in V \mid \{u,v\} \in E \text{ and } \tau(u^1) = 0 \}$. Assume that:
\begin{itemize}
\item there exists $z^1 \in V^1$ with $z^1 \leq u$ and $\tau(z^1) = 0$ and, 
\item for every $v^2 \in V^2$ with $v^2 \leq u$, we have  $\tau(v^2) = 1$,
\item $(V^2 \cap [>u]) \setminus W^2 \neq \emptyset$.
\end{itemize}
There exists a node $\alpha$ in $Z$ such that $\alpha$ is reached by $\tau$, is maximal and $\var(Z_\alpha) \supseteq (V^2 \cap [>u]) \setminus W^2$ where $W = \{v \in V \mid \{u,v\} \in E \text{ and } \tau(u^1) = 0 \}$.
\end{lemma}
\begin{proof}
  Assume toward a contradiction that there exists a node $\alpha$ reached by $\tau$ such that $\var(Z_\alpha) \cap (V^2 \cap [>u]) \setminus W^2 \neq \emptyset$ and $x \in (V^2 \cap [>u]) \setminus W^2$ with $x \notin \var(Z_\alpha)$ and let $y \in \var(Z_\alpha) \cap (V^2 \cap [>u]) \setminus W^2$. Let $\tau_1$ be the assignment such that $\tau_1(x) = 0$ and for every $z > u$ such that $z \neq x$, $\tau_1(z) = 1$. Let $\tau_2$ be the assignment such that $\tau_2(y) = 0$ and for every $z > u$ such that $z \neq y$, $\tau_2(z) = 1$. 

It is easy to see that $\tau \cup \tau_1$ satisfies $F^2_G$. Indeed, $C_1$ is satisfied by $\tau$ since by assumption $\tau(z^1) = 0$ and $C_2$ is satisfied since $\tau_1(x) = 0$ and $x \in V^2$. Now, let $\{v,w\} \in E$. Observe that the only variable of $V^2$ set to $0$ by $\tau_1$ is $x$. Thus, if $w \neq x$, $v^1 \vee w^2$ is satisfied. Now, assume $w = x$. If $v^1 > u$, then by definition $\tau_1(v^1) = 1$ thus $v^1 \vee w^2$ is satisfied. Now if $v^1 \leq u$, then since $x \notin W^2$, by definition of $W^2$, $v^1$ is not assigned to $0$ by $\tau$. Thus, $\tau(v^1) = 1$ and $v^1 \vee w^2$ is satisfied. Similarly, $\tau \cup \tau_2$ satisfies $F^2_G$. 

By Lemma~\ref{lem:rectangle}, $\tau' = \tau \cup \tau_1|_{\var(Z_\alpha)} \cup \tau_1|_{[>u] \setminus \var(Z_\alpha)} \models F^2_G$. However, it is straightforward to see that $\tau'(x) = \tau'(y) = 1$. And thus, for every $v^2 \in V^2$, $\tau'(v^2) = 1$, thus $\tau'$ does not satisfy $C_2$. Contradiction.

We have shown so far that if a node $\alpha$ is reached by $\tau$, either $\var(Z_\alpha) \supseteq (V^2 \cap [>u]) \setminus W^2$ or $\var(Z_\alpha) \cap (V^2 \cap [>u]) \setminus W^2 = \emptyset$. It thus remains to prove that there exists a maximal node $\alpha$ reached by $\tau$ such that $\var(Z_\alpha) \cap (V^2 \cap [>u]) \setminus W^2 \neq \emptyset$. Let $M$ be the set of nodes reached by $\tau$ that are maximal and assume that for every $\alpha \in M$, $\var(Z_\alpha) \cap (V^2 \cap [>u]) \setminus W^2 = \emptyset$. By Lemma~\ref{lem:reachable}, we have $Z[\tau] = \bigwedge_{\alpha \in M} Z_\alpha$. That is, $Z[\tau]$ does not depend on $(V^2 \cap [>u]) \setminus W^2$. Let $\tau' : [>u] \rightarrow \{0,1\}$ be such that for every $v > u$, $\tau'(v) = 1$. By definition, for every $v \in V$, $(\tau \cup \tau')(v^2) = 1$. Thus $\tau \cup \tau' \not \models F^2_G$ since $C_2$ is not satisfied by $\tau \cup \tau'$. However, we claim that $\tau' \models Z[\tau]$. Indeed, let $v^2 > u$ be such that $v^2 \notin W^2$ which exists since $(V^2 \cap [>u]) \setminus W^2 \neq \emptyset$. Let $\tau''$ be the assignment that differs from $\tau'$ only on $v^2$. It is readily verified that $\tau \cup \tau'' \models Z$ thus $\tau'' \models Z[\tau]$. However, since $Z[\tau]$ does not depend on $v^2$, $\tau' \models Z[\tau]$ too. Contradiction.
\end{proof}

We are ready to prove a general lower bounds on $\adobdd_<$ computing $F^2_G$. 

\begin{theorem}
  \label{thm:genericlb} Let $G = (V,E)$ be a graph, $<$ an order on $V^1 \cup V^2$ and $Z$ be an $\adobdd_<$ for $F^2_G$. Assume there exists $u \in V^1 \cup V^2$ and a non-empty set $M \subseteq V^1 \times V^2$ such that:
  \begin{itemize}
  \item for every $(v^1,w^2) \in M$, $v^1 \leq u < w^2$ and,
  \item $E(M) = \{\{v,w\} \mid (v^1,w^2) \in M\}$ is an induced matching of $G$,
  \end{itemize}
We have $|Z| \geq 2^{|M|-2}$.
\end{theorem} 
\begin{proof}
If $|M| \leq 2$, the result is trivial. Otherwise, we arbitrarily pick two distinct edges $(a^1,b^2), (c^1,d^2)\in M$ and let $M' = M \setminus \{(a^1,b^2), (c^1,d^2) \}$. Let $U_1$ be the variables of $V^1$ that are in $M'$. That is $U_1 = \{x^1 \mid \exists y^2 (x^1,y^2) \in M'\}$. Observe that since $E(M')$ is a induced matching of $G$ of size $|M|-2$, we have $|U_1| = |M|-2$.

Given $A \subseteq U^1$, we define $\tau_A: [\leq u] \rightarrow \{0,1\}$ as follows: 
\begin{itemize}
\item for every $v^1 \in \{a^1\} \cup A$, $\tau_A(v^1) = 0$,
\item $\tau(c^1) = 1$ and,
\item for every other $v \leq u$, $\tau_A(v) = 1$.
\end{itemize}
Since $\tau_A(a^1) = 0$, there exists $v^1 \leq u$ such that $v^1 \in V^1$ and $\tau_A(v^1) = 0$. Moreover, since $A \subseteq V^1$, for every $v^2 \in V^2 \cap [\leq u]$, $\tau_A(v^2) = 1$. Finally, $c^1 \leq u < d^2$ and since $E(M)$ is an induced matching, every neighbour $n$ of $d$ such that $n^1 \leq u$ verifies  $\tau_A(n^1) = 1$.  Thus, we can apply Lemma~\ref{lem:extractnode}. Let $\alpha_A$ be the node given by Lemma~\ref{lem:extractnode} applied to $Z$ and $\tau_A$.

We claim that if $A \neq B$, then $\alpha_A \neq \alpha_B$. Since $U_1 = |M'| = |M|-2$, the theorem follows since it gives $2^{|M|-2}$ different nodes in $Z$.

We now prove that if $A \neq B$, then $\alpha_A \neq \alpha_B$. Assume toward a contradiction that $\alpha_A = \alpha_B = \alpha$. Without lost of generality, assume there exists $v^1 \in A \setminus B$ and let $w^2$ be the only vertex of $W^2$ such that $(v^1,w^2) \in M'$ (the uniqueness of $w^2$ follows from the fact that $E(M)$ is a matching). By definition, $\tau_A(v^1) = 0$ and $\tau_B(v^1) = 1$. Moreover, since $E(M)$ is an induced matching, $w$ is not in the neighbourhood of any vertex of $M$ but $v$. Thus, it holds that $\tau_B(x^1) = 1$ for every neighbour $x$ of $w$ with $x^1 \leq u$. By Lemma~\ref{lem:extractnode}, $w^2 \in \var(Z_{\alpha_B})$.

Now, let $\tau_B':[>u]\rightarrow\{0,1\}$ be defined as follows: $\tau_B'(d^2) = \tau_B'(w^2) = 0$ and for every other variables $u'$ of $[>u]$, $\tau_B'(u') = 1$. It is easy to see that $\tau_B \cup \tau_B' \models F^2_G$, thus by Lemma~\ref{lem:reachable}, $\tau_B'|_{X_B} \models Z_{\alpha_B}$ where $X_B = \var(Z_{\alpha_B})$.

Finally, let $\tau'_A:[>u]\rightarrow\{0,1\}$ be defined as follows: $\tau'_A(d^2) = 0$ and for every other variables $u'$ of $[>u]$, $\tau'_A(u') = 1$. It is easy to see that $\tau_A \cup \tau'_A \models F^2_G$. If $\alpha_A = \alpha_B$ then  $\tau_A \cup \tau'_A$ reaches $\alpha_B$. By Lemma~\ref{lem:rectangle}, $\tau' = \tau_A \cup \tau'_B|_{X_B} \cup \tau'_A|_{[>u] \setminus X_B} \models F^2_G$.  However, $\tau'(v^1) = \tau_A(v^1) = 0$ and $\tau'(w^2) = \tau_B'(w^2) = 0$ meaning that $\tau'$ does not satisfy the clause $\{v^1,w^2\}$, contradiction. 
\end{proof}

\subsection{Instances of small incidence treewidth}
\label{sec:inctw}

In this section, we finally prove Theorem~\ref{thm:aobddlb} by constructing a family of graphs of incidence treewidth at most $k$ such that for any graph $G = (V,E)$ of this family and any order $<$ on $V^1 \cup V^2$, we can find a set $M$ of size $\Omega(k \log(n))$ as in the statement of Theorem~\ref{thm:genericlb}.

The following has been shown in~\cite{Razgon14}:
\begin{theorem}[Theorem~3 in~\cite{Razgon14}]
\label{thm:razgongraphs}
There exists a constant $b$ such that for every $k$, there is an infinite class of graphs $\calG_k$ of degree at most $5$ and treewidth at most $k$ such that for every $G = (V,E)$ and linear order $<$ on $V$, there exists $u \in V$ and a matching $M$ of size at least $k\log(|V|)/b$ and such that for every $\{v,w\} \in M$ with $v < w$, we have $v \leq u < w$.
\end{theorem}

\begin{lemma}
\label{lem:obddtw}
  Let $G$ be a graph of treewidth $k$. We have $\itw(F_G^2) \leq 2k+3$. 
\end{lemma}
\begin{proof}
Take a tree decomposition for $G$ of width $k$.
That is, the size of each bag is at most $k+1$.
For each $u \in V(G)$, replace the occurrence of $u$ in each
bag by $u^1$ and $u^2$. Then add the occurrences of negative clauses
to each bag. It is verifiable by a direct inspection that we obtain
a tree decomposition of the incidence graph of $F^2_{G}$
where the size of each bag is at most $2k+4$. Hence the width of the tree
decomposition is at most $2k+3$.  
\end{proof}

We are now ready to prove Theorem~\ref{thm:aobddlb}.
\begin{proof}[Proof (of Theorem~\ref{thm:aobddlb}).]
For $k \geq 2$, let $k' = \lfloor (k-3)/2 \rfloor$ and  let $\calF_k = \{F_G^2 \mid G \in \calG_{k'}\}$ where $\calG_k$ is the class of graphs from Theorem~\ref{thm:razgongraphs}. By Lemma~\ref{lem:obddtw}, we have that every formula of $\calF_k$ is of incidence treewidth at most $2k'+3 \leq k$.

Let $G=(V,E)$ be a graph and $<$ be an order on $V^1 \cup V^2$, we denote by $\prec$ the order on $V$ defined as follows: for every $u,v \in V$, $u \prec v$ if and only if $\min(u^1,u^2) < \min(v^1,v^2)$, that is, $u \prec v$ iff the first copy of $u$ comes before the first copy of $v$.

Let $u \in V$ and $M$ be an induced matching such that for every $\{v,w\} \in M$ with $v \prec w$, we have $v \preceq u \prec w$. Let $A_1 = \{v \mid \{v,w\} \in M \text{ and } v \prec w\}$ and $A_2 = \{v \mid \{v,w\} \in M \text{ and } v \prec w\}$. By definition of $\prec$, there exists $i \in \{1,2\}$ such that for every $\{v,w\} \in M$ with $v \prec w$, either $v^1 \leq u^i < w^2$ or $v^2 \leq u^i < w^1$.  Let $M' = \{(v^1,w^2) \mid \{v,w\} \in M, v^1 \leq u^i < w^2 \}$. We assume that $|M'| > |M|/2$. If it is not the case, we can enforce it by only changing the roles of colors $1$ and $2$. It is readily verified that $M'$ verifies the conditions of Theorem~\ref{thm:genericlb} and thus, any $\adobdd_<$ is of size at least $2^{|M|/2-1}$. 
  
Now, if $G$ is a graph from $\calG_{k'}$, then we know that such a matching $M$ exists and is of size $\Omega(k'\log(n)) = \Omega(k\log(n))$. Moreover, since $G$ is of degree at most $5$, by using Lemma~\ref{lem:extractinduced}, we can extract an induced matching from $M$ of size $\Omega(k \log(n))$ too, leading to an $n^{\Omega(k)}$ lower bound on the size of $\adobdd$ for $F_G^2$.
\end{proof}

\section{Future research}
\label{sec:conclusion}

In this paper, we have shown that two restrictions of decDNNF cannot represent instances of bounded incidence treewidth efficiently, that is, in FPT-size. The question of whether decDNNF can efficiently represent instances of bounded incidence treewidth is still open and proving a non-FPT lower bound in this case would likely require new techniques to be developed. Indeed, if $G$ is of bounded treewidth, then both $F_G$ and $F^2_G$ can be represented by FPT-size decDNNF. 

Another interesting and related question would be to understand the proof complexity of CNF-formulas of bounded incidence treewidth. It is indeed still open whether unsatisfiable CNF-formulas of incidence treewidth $k$ have a resolution refutation of FPT-size. 
% \begin{itemize}
% \item Lower bound on decDNNF? (Mention FPT algo for bounded number of clauses).
% \item Toward lower bound on resolution (ie SAT Solvers)?
% \end{itemize}

\bibliography{/home/fcapelli/biblio.bib}

\end{document}

%% file: decdnnf.tikz
\node (x) at (0,0) {$\pmb x$};
\node (a) at (-1,-1) {$\pmb{\wedge}$};
\node (y1) at (1,-1) {$y$};
\node (y2) at (-2,-2) {$\pmb y$};
\node (z2) at (0,-2) {$\pmb z$};
\node (z1) at (2,-2) {$z$};
\node (f) at (-1,-3) {$0$};
\node (t) at (1,-3) {$\pmb 1$};

\draw[->,dashed,very thick] (x) -- (a);
\draw[->,very thick] (a) -- (z2);
\draw[->,very thick] (a) -- (y2);
\draw[->] (x) -- (y1);
\draw[->,dashed] (y1) -- (z2);
\draw[->] (y1) -- (z1);
\draw[->,very thick] (y2) -- (t);
\draw[->,dashed] (y2) -- (f);
\draw[->,very thick] (z2) -- (t);
\draw[->,dashed] (z2) -- (f);
\draw[->] (z1) -- (f);
\draw[->,dashed] (z1) -- (t);

%% file: sdec.tikz
\node (x) at (0,0) {$x$};
\node (y1) at (-2,-1) {$y$};
\node (z1) at (-2,-2) {$z$};

\draw[-latex,bend right, dashed] (y1) edge (z1);
\draw[-latex,bend left] (y1) edge (z1);

\node (z2) at (2,-1) {$z$};
\node (y2) at (2,-2) {$y$};

\draw[-latex,bend right, dashed] (z2) edge (y2);
\draw[-latex,bend left] (z2) edge (y2);

\draw[dashed,->](x) -- (y1);
\draw[->](x) -- (z2);

\node (t) at (-1,-3) {$1$};
\node (f) at (1,-3) {$0$};

\draw[->,dashed] (z1) -- (t);
\draw[->] (z1) -- (f);

\draw[->,dashed] (y2) -- (t);
\draw[->] (y2) -- (f);